\documentclass[letterpaper, 10 pt, conference]{ieeeconf}  

\IEEEoverridecommandlockouts                              

\usepackage{amsmath}
\usepackage{esint}
\usepackage{amssymb}
\usepackage{amsthm}                                                                          
\usepackage{mathtools}
\usepackage{derivative}
\usepackage{xcolor}
\usepackage{bm}
\usepackage{physics} 
\usepackage[ruled,vlined]{algorithm2e}
\usepackage{optidef}
\usepackage{graphicx}   
\usepackage{caption}    
\usepackage{subcaption} 

\usepackage[noadjust]{cite}
\usepackage{url}
\usepackage[hidelinks]{hyperref}

\usepackage{graphicx}
\usepackage[export]{adjustbox} 
\graphicspath{{figures/}}
\usepackage{stfloats}

\newtheorem{theorem}{Theorem}

\theoremstyle{definition}
\newtheorem{definition}{Definition}

\newtheorem{remark}{Remark}

\usepackage{bm}

\DeclareMathOperator*{\argmin}{arg\,min}

\title{\LARGE \bf
Full-Body Dynamic Safety for Robot Manipulators: \\ 
3D Poisson Safety Functions for CBF-Based Safety Filters
 }
\author{
Meg Wilkinson$^{1}$, Gilbert Bahati$^{2}$, Ryan M. Bena$^{2}$, Emily Fourney$^{1}$,
Joel W. Burdick$^{1,2}$, Aaron D. Ames$^{1,2}$%
\thanks{$^{1}$Computing and Mathematical Sciences, California Institute of Technology. $^{2}$Mechanical and Civil Engineering, California Institute of Technology.}%
\thanks{Email: {\tt\small \{mwilkins,gbahati,ryanbena,efourney,\protect\\
burdick,ames\}@caltech.edu}}}

\begin{document}

\maketitle
\thispagestyle{empty}
\pagestyle{empty}




\newcommand{\re}{\mathbb{R}}


\newcommand{\dt}{\mathrm{d}t}
\newcommand{\dy}{\mathrm{d}y}
\newcommand{\dx}{\mathrm{d}x}
\newcommand{\dtau}{\mathrm{d}\tau}
\newcommand{\Cc}{\mathcal{C}}
\newcommand{\Cce}{\mathcal{C}_\varepsilon}
\newcommand{\he}{h_{\varepsilon}}
\newcommand{\C}{\mathcal{C}}
\newcommand{\Ac}{\mathcal{A}}
\newcommand{\pCc}{\partial \mathcal{C}}
\newcommand{\Bc}{\mathcal{B}}
\newcommand{\Tc}{\mathcal{T}}
\newcommand{\Dc}{\mathcal{D}}
\newcommand{\Oc}{\Omega}
\newcommand{\Occ}{\overline{\Omega}}
\newcommand{\pOc}{\partial \Omega}
\newcommand{\Ocext}{\Oc_\mathrm{ext}}
\newcommand{\Ocint}{\Oc_\mathrm{int}}
\newcommand{\Hc}{\mathcal{H}}
\newcommand{\Kc}{\mathcal{K}}
\newcommand{\Fc}{\mathcal{F}}
\newcommand{\Mc}{\mathcal{M}}
\newcommand{\Nc}{\mathcal{N}}
\newcommand{\Pc}{\mathcal{P}}
\newcommand{\Uc}{\mathcal{U}}
\newcommand{\Sc}{\mathcal{S}}
\newcommand{\Xc}{\mathcal{X}}
\newcommand{\Yc}{\mathcal{Y}}
\newcommand{\Vc}{\mathcal{V}}
\newcommand{\Zc}{\mathcal{Z}}
\newcommand{\Lc}{\mathcal{L}}
\newcommand{\Rm}{\mathcal{\mathbb{R}}}
\newcommand{\R}{\mathcal{\mathbb{R}}}
\newcommand{\Sp}{\mathcal{\mathbb{S}}}

\newcommand{\divv}{\nabla \cdot \vec{\bv}}
\newcommand{\hs}{h_\mathrm{\Sc}}

\newcommand{\vr}{\varepsilon}
\newcommand{\nom}{{\operatorname{nom}}}
\newcommand{\m}{{\operatorname{min}}}
\newcommand{\des}{{\operatorname{des}}}
\newcommand{\on}{{\operatorname{on}}}
\newcommand{\off}{{\operatorname{off}}}
\newcommand{\fl}{{\operatorname{FL}}}
\newcommand{\Lie}{\mathcal{L}}
\newcommand{\qp}{{\operatorname{QP}}}

\newcommand{\ie}{i.e., }
\newcommand{\todo}[1]{{\color{cyan} Todo: #1}}

\newcommand{\ba}{\mathbf{a}}
\newcommand{\bb}{\mathbf{b}}
\newcommand{\be}{\mathbf{e}}
\renewcommand{\bf}{\mathbf{f}} 
\newcommand{\bff}{\mathbf{f}}
\newcommand{\bg}{\mathbf{g}}
\newcommand{\bk}{\mathbf{k}}
\newcommand{\bp}{\mathbf{p}}
\newcommand{\bq}{\mathbf{q}}
\newcommand{\bu}{\mathbf{u}}
\newcommand{\bv}{\mathbf{v}}
\newcommand{\bvv}{\vec{\mathbf{v}}}
\newcommand{\bn}{\mathbf{n}}
\newcommand{\hbn}{\hat{\mathbf{n}}}

\newcommand{\bx}{\mathbf{x}}
\newcommand{\bz}{\mathbf{z}}
\newcommand{\br}{\mathbf{r}}
\newcommand{\bA}{\mathbf{A}}
\newcommand{\bB}{\mathbf{B}}
\newcommand{\bD}{\mathbf{D}}
\newcommand{\bC}{\mathbf{C}}
\newcommand{\bF}{\mathbf{F}}
\newcommand{\bJ}{\mathbf{J}}
\newcommand{\bG}{\mathbf{G}}
\newcommand{\bK}{\mathbf{K}}
\newcommand{\bP}{\mathbf{P}}
\newcommand{\bW}{\mathbf{W}}
\newcommand{\bw}{\mathbf{w}}
\newcommand{\bd}{\mathbf{d}}
\newcommand{\bvy}{\vec{\by}}
\newcommand{\bty}{\tilde{\by}}
\newcommand{\bbeta}{\boldsymbol{\eta}}
\newcommand{\mb}[1]{\mathbf{#1}}

\newcommand{\bY}{\mathbf{Y}}
\newcommand{\by}{\mathbf{y}}
\newcommand{\byobs}{\mathbf{y}_\mathrm{obs}}

\newcommand{\bxd}{\bx_\mathrm{d}}
\newcommand{\bxobs}{\bx_\mathrm{obs}}
\newcommand{\md}{\mathrm{d}}

\newcommand{\Uxd}{U_{\mathrm{d}}}
\newcommand{\Uobs}{U_{\mathrm{obs}}}
\newcommand{\Uapf}{U_{\mathrm{APF}}}

\newcommand{\GradUxd}{\nabla U_{\mathrm{d}}}
\newcommand{\GradUobs}{\nabla U_{\mathrm{obs}}}
\newcommand{\GradUapf}{\nabla U_{\mathrm{APF}}}

\newcommand{\cmax}{c_\mathrm{max}}
\newcommand{\cmin}{c_\mathrm{min}}

\newcommand{\hn}{h_\mathrm{n}}
\newcommand{\Dhn}{D h_\mathrm{n}}
\newcommand{\Dh}{D h}
\newcommand{\Dhd}{D h_\mathrm{d}}

\begin{abstract}

Collision avoidance for robotic manipulators requires enforcing full-body safety constraints in high-dimensional configuration spaces.
Control Barrier Function (CBF) based safety filters have proven effective in enabling safe behaviors, but enforcing the high number of constraints needed for safe manipulation leads to theoretic and computational challenges. 
This work presents a framework for full-body collision avoidance for manipulators in dynamic environments by leveraging 3D Poisson Safety Functions (PSFs).
In particular, given environmental occupancy data, we sample the manipulator surface at a prescribed resolution and shrink free space via a Pontryagin difference according to this resolution. On this buffered domain, we synthesize a globally smooth CBF by solving Poisson's equation---yielding a single safety function for the entire environment.
This safety function, evaluated at each sampled point, yields task-space CBF constraints enforced by a real-time safety filter  via a multi-constraint quadratic program. 
We prove that keeping the sample points safe in the buffered region guarantees collision avoidance for the entire continuous robot surface. The framework is validated on a 7-degree-of-freedom manipulator in dynamic environments.




\end{abstract}

\section{INTRODUCTION}

Robotic manipulators operating in unstructured and dynamic environments must avoid collisions---not just at the end-effector, but over the entire robot body. This full-body requirement is fundamental, yet achieving provably safe full-body collision avoidance in real time remains challenging.  
This is primarily due to the fact that the controller must operate in high-dimensional joint spaces, while enforcing a prohibitively large number of safety constraints---ensuring collision avoidance between every point on the robot and the environment. Planning safe manipulator trajectories around dynamic obstacles has been well explored \cite {doi:10.1177/027836498500400308}, 
but there remains computational and theoretical bottlenecks in achieving dynamic safety in unknown \textit{a priori} environments.


Control Barrier Functions (CBFs) \cite{ADA-JWG-PT:14} offer an attractive framework for this problem due to their efficient online computational benefits. Given a safe set that encodes collision-free configurations, a CBF-based safety filter minimally modifies a nominal control input by solving a quadratic program (QP) to guarantee forward invariance of the safe set. This reactive approach---which adjusts commands only when necessary and preserves the nominal behavior otherwise---has been applied to manipulators \cite{SingletaryMolnarSafetyCriticalFood}, and a variety of robotic platforms \cite{bahati2025control, TamasRAL22, grandia2020nonlinear,cohen2025safety, Khazoom22}. For manipulators specifically, \cite{SingletaryMolnarSafetyCriticalFood} demonstrated real-time CBF-based trajectory modification for food preparation, and proved that safety guarantees established at the kinematic level extend to the full-order dynamics under accurate velocity tracking.
Yet this work introduced conservatism to account for non-smoothness of the CBF-based safety filter.  

\begin{figure}[t!]
\centering
    \includegraphics[width=0.98\linewidth]{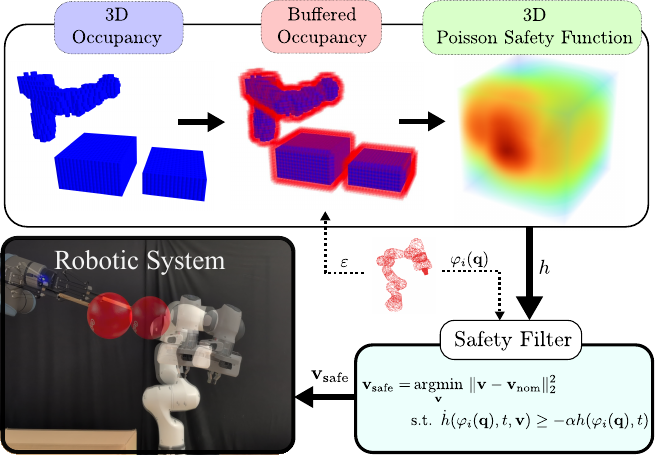}
    \caption{\small Full-body collision avoidance with Poisson-based CBF Safety Filters.  An occupancy map is buffered by the sampling resolution $\varepsilon$. Solving Poisson's equation on this domain generates a globally smooth safety function. A multi-constraint CBF-QP is composed of a single PSF evaluated at the sample points, enabling full body collision avoidance in dynamic environments.  }
    \label{fighero_fig}
    \vspace{-6mm}
\end{figure}
This paper addresses two challenges that arise when extending CBFs to full-body manipulator safety. The first challenge is that CBF-based safety constraints must be continuously differentiable (\ie ideally smooth) as the CBF-QP requires well-defined gradients at every point where a constraint is evaluated. Signed distance functions (SDFs), while geometrically natural, lose differentiability within safe regions where their gradients become ill-conditioned. Recent work has addressed this problem through post-processing methods that yield smooth approximations of SDFs \cite{SingletaryMolnarSafetyCriticalFood, zhao2026sq}, or through operational-space formulations that scale to hundreds of analytically defined constraints \cite{morton2025safe}. However, these approaches require ad-hoc geometric primitives (spheres, rectangles, ellipses) that over-approximate the true robot or obstacle geometry, introducing conservatism, and require a separate barrier construction for each obstacle. 

The second challenge is that enforcing safety at finitely many body points does not, by itself, guarantee safety for the continuous surface between those points---a gap that existing multi-constraint CBF formulations \cite{SingletaryMolnarSafetyCriticalFood, zhao2026sq, Khazoom22} leave unaddressed. While \cite{ding_2024_online} provides probabilistic bounds on safety, it does not provide full body safety guarantees. CBFs defined independently for each safety objective or sample point can yield a large number of constraints. Enforcing all these constraints in the QP can quickly result in infeasibility. A solution used in practice is to introduce a slack variable to the safety constraints and penalizing these violations heavily in the objective function. While relaxation improves numerical feasibility and performance, it weakens the forward invariance guarantees. Learning based pipelines also lack these guarantees, but yield impressive success rates in static and dynamic environments  \cite{yang2025deep}. However, these can be prone to failing unpredictably out of the training distribution \cite{fishman2022motionpolicynetworks}. 



We present a framework for full-body
collision avoidance of manipulators in dynamic environments by addressing the above challenges simultaneously. For continuously differentiable safety constraints, we employ Poisson Safety Functions (PSFs) \cite{bahati2025dynamic}---functional safety representations obtained by solving an elliptic partial differential equation (PDE), Poisson's equation, on the obstacle-free region provided by occupancy data. Solving this PDE produces a single globally smooth safety function for the entire environment, whose zero level set is aligned to obstacle surfaces by construction. Unlike SDFs, PSFs are smooth everywhere in the interior (by elliptic regularity \cite{gilbarg1977elliptic}), and unlike primitive-based methods, they represent arbitrary obstacle geometry without over-approximation. Furthermore, PSFs allow for temporal parameterizations to account for dynamic environments \cite{bena2025geometry, erina2026} and semantic-aware control synthesis \cite{bahati2025risk}.

To achieve full-body safety, we propose a sampling-and-buffering method with provable full-body safety guarantees by sampling the manipulator surface, and shrinking the safe region via a Pontryagin difference according to the sampling resolution, and prove that keeping the sample points safe in the buffered region implies safety for the entire continuous surface. This unified safety representation of the environment avoids conflicting constraints and eliminates infeasibility issues in the QP, allowing for a high number of hard safety constraints. %
Our contributions are summarized as follows:

\begin{enumerate}
    \item \textit{PSF-based CBF safety filter for 
    $n$-DOF manipulators}:
    We show how a single smooth task-space PSF, evaluated at finitely many body points, achieves real-time safety on $n$-DOF manipulators in dynamic scenarios.

    \item \textit{Formal guarantee for full-body collision avoidance:}
    We prove that sampling the robot surface at a prescribed resolution and buffering the obstacles by that resolution is sufficient to  guarantee forward invariance of the full robot body. 

    \item \textit{Hardware validation.}
    We demonstrate this pipeline on hardware on a 7-degree-of-freedom manipulator in dynamic environments.
\end{enumerate}



\section{Safety With Control Barrier Functions}
\subsection{Background: Control Barrier Functions}
Consider the nonlinear control-affine system of the form:
\begin{align}\label{eq: nl}
    \dot{\bx} = \bf(\bx) + \bg(\bx) \bu,
\end{align}
where $\bx \in \re^n$ is the system's state, $\bu \in \re^m$ is the control input, and $\bf:\re^n \rightarrow \re^n$, $\bg:\re^n \rightarrow \re^{n \times m}$ are the drift and actuation terms, respectively, both assumed to be locally Lipschitz continuous. A locally Lipschitz continuous control law $\bk: \re^n \times \re_{+ }\rightarrow \re^m$ yields a closed loop system $ \dot{\bx} = \bf(\bx) + \bg(\bx) \bk(\bx,t)$ for which, given an initial condition $\bx(0) = \bx_0$, admits a unique solution $t \mapsto \bx(t)$, which we assume exist for all $t >0$ \cite{perko2013differential}.


Our goal is to construct a controller that keeps the trajectories $t \mapsto \bx(t)$ within a \textit{safe set} for all $t>0$. In manipulation tasks, the safe set is typically encoded by subsets of unoccupied space. In particular, we consider a time-varying safe set $\Cc_t \subset \re^n$ defined as the $0$-super level set of a continuously differentiable function $h:\re^n \times \re_+ \rightarrow \re$ defined as:
\begin{align}\label{eq: safe set}
    \Cc_t = \{ \bx \in \re^n : h(\bx,t) \geq 0
    \}.
\end{align}
Achieving safety implies keeping value of $h$ positive along the trajectories $t \mapsto \bx(t) \ \forall t>0$, formalized by the notion of \textit{forward invariance}. Control Barrier Functions (CBFs) provide a constructive tool for synthesizing controllers that enforce forward invariance of safe sets.

\begin{definition}(Control Barrier Functions~\cite{ADA-XX-JWG-PT:17}) Let $\Cc_t \subset \re^n$ be the time-varying $0$-super-level set of a continuously differentiable function  $h:\re^n \times \re_+ \rightarrow \re$ satisfying $\nabla h(\bx,t) \neq \mathbf{0}$ when $h(\bx,t) =0$. We call $h$ a time-varying Control Barrier Function (CBF) for \eqref{eq: nl} if there exists\footnote{\small A continuous function $\alpha : \re \to \re$ belongs to the extended class $\Kc_{\infty}^{e}$ if it is monotonically increasing, $\alpha(0) = 0$, $\lim_{s \to \infty} \alpha(s) = \infty$, and $\lim_{s \to -\infty} \alpha(s) = -\infty$.}
 $\alpha \in \Kc_{\infty}^{e}$ such that for all $(\bx,t) \in \re^n \times \re_+$:
\begin{align}\label{CBF condition}
    \sup_{\bu \in \re^m} \dot h(\bx,t,\bu) \stackrel{\Delta}{=} \nabla &h (\bx,t) \cdot (\bf(\bx) + \bg(\bx) \bu) \nonumber \\& + \frac{\partial h}{\partial t}(\bx,t) \geq - \alpha(h(\bx,t)).
\end{align}
\end{definition}

Given a nominal controller $\bk_\mathrm{nom}:\re^n \times \re_+ \rightarrow \re^m$, one standard way of constructing a safe controller is via the CBF-QP that adjusts $\bk_\mathrm{nom}$ to its nearest safe action:
\begin{align*}\label{eq: safety filter}
    \bk(\bx,t) = &\argmin_{\bu \in \re^m} &&\|\bu - \bk_{\mathrm{nom}}(\bx,t)\|_2^2 \tag{CBF-QP} \\
    & \quad 
 \mathrm{s.t.} && \dot h(\bx,t,\bu)   \geq - \alpha(h(\bx,t)).
\end{align*}

\subsection{Application to Manipulators}

Robotic manipulators are high degree of freedom (DOF) systems that require safety guarantees across their full 3D geometric structure. 
%
%
Consider an $n$-DOF manipulator whose state $\bx = (\mathbf{q}, \dot{\mathbf{q}})$ consists of configurations $\bq \in \re^n$ and joint velocities $\dot{\bq} \in \re^n$, governed by the dynamics:
\begin{align}
\label{eq: mechanical system}
M(\mathbf{q})\ddot{\mathbf{q}} + C(\mathbf{q}, \dot{\mathbf{q}})\dot{\mathbf{q}} +G(\mathbf{q}) = B\mathbf{u},
\end{align}
where $M(\mathbf{q}) \in \mathbb{R}^{n \times n}$, $C(\mathbf{q}, \dot{\mathbf{q}})\in \mathbb{R}^{n \times n}$, $G(\mathbf{q})\in \mathbb{R}^{n}$ denote the inertia matrix, Coriolis matrix, and gravity terms, respectively, $\mathbf{u} \in \mathbb{R}^m$ is the control input (\ie joint torques), and $B \in \re^{n \times n}$ is invertible, assuming full actuation.

Motivated by \cite{SingletaryMolnarSafetyCriticalFood}, we interpret safety-critical control at the \textit{kinematic} level. 
%
That is, we consider the system:
 \begin{align}\label{eq: velocity model}
     \dot{\bq} = \bv,
 \end{align}
wherein we assume direct control over the joint velocities via
the commanded velocity $\bv \in \re^n$. We design safe velocities for \eqref{eq: velocity model}, and, as verified in \cite{SingletaryMolnarSafetyCriticalFood}, these safety guarantees extend to the full dynamics \eqref{eq: mechanical system} provided the commanded velocity is
accurately tracked by a low-level controller.

Collision avoidance specifications for manipulators are typically described in the task space (\ie 3D Euclidean space $\re^3$) while the control input acts in the joint (configuration) space $\re^n$, related via a kinematic map $\varphi : \re^n \rightarrow \re^3$.
%
 %
For a point $\by \in \re^3$ on the manipulator body, the forward kinematics relate the joint and task space coordinates via  $\by = \varphi(\bq)$, yielding 
$\dot{\by} = \frac{\partial \varphi}{\partial \bq}(\bq) \dot{\bq} \coloneq J(\bq)\dot{\bq},
$
where $J(\bq)   \in \re^{3 \times n}$ is the corresponding Jacobian. 
Given a continuously differentiable function $h:\re^3 \times \re_+ \rightarrow \re$ whose $0$-superlevel set defines collision-free space,
%
we define a joint-space safety function by the composition 
%
%
%
$h(\varphi(\bq),t) \coloneqq h(\by,t)$ with corresponding safe set in configuration space:
\begin{align}\label{eq: q safe set}
\Cc_\bq(t) = \{ \bq \in \re^n \, | \, h(\varphi(\bq),t) \geq 0 \}.
\end{align}
The time derivative of $h$ is given by:
\begin{align}
\dot{h}(\varphi(\bq), t, \bv) &= \frac{\partial h}{\partial \by}(\varphi(\bq),t) \cdot J(\bq) \bv + \frac{\partial h}{\partial t}(\varphi(\bq),t)\label{eq: task space gradients}\\
&\stackrel{\Delta}{=}\frac{\partial h}{\partial \bq}(\bq,t) \cdot \bv + \frac{\partial h}{\partial t}(\bq,t).
\end{align}
%
%
%
Safe joint velocities are synthesized by solving a CBF-QP.
Since the robot is a continuous surface, directly enforcing safety at every point is computationally intractable. Instead, safety is typically enforced point-wise on a \textit{collection} of $N$ sample points on the robot body 
$\by_i = \varphi_i(\bq)$, $i \in \{1, \ldots, N\}$, each querying the barrier function $h$, yielding a multi-constraint safety filter:
%
%
\begin{align*}
{\bv}_\mathrm{safe} = &\argmin_{\bv \in \re^n} \quad \|{\bv}- {\bv}_{\nom}\|_2^2 \tag{MC-CBF-QP} \\
& \quad \mathrm{s.t.} \quad 
 \dot{h}(\varphi_i(\bq), t,\bv)  \geq - \alpha_i (h(\varphi_i(\bq),t)),
%
\end{align*}
for all $i \in \{1, \cdots, N\}$ and $\alpha_i \in \Kc^e_{\infty}$, where $\bv_\nom$ is a desired, nominal joint velocity command, and
$\frac{\partial h}{\partial \bq}(\bq,t)
=\frac{\partial h}{\partial \by}(\varphi_i(\bq),t) \cdot J_i(\bq)$. 

%
%

The above formulation exposes two requirements that are difficult to satisfy simultaneously. 
First, the \textit{single} task-space safety function $h$ must be continuously differentiable to guarantee well-defined gradients in \eqref{eq: task space gradients} for the QP, yet common representations such as signed distance functions (SDFs) lose differentiability in regions where the nearest obstacle is not unique~\cite{SingletaryMolnarSafetyCriticalFood}. Additionally, smooth ad-hoc barrier constructions (i.e., analytical expressions where ``obstacles” are described by implicit
geometric shapes, e.g., circles, ellipses, polygons) require a separate barrier function for each obstacle, and fail to accurately represent obstacles of arbitrary geometry. Their composition into a single $h$  struggles to generalize to complex environments without introducing excessive conservatism \cite{molnar2023composing}.
Second, the MC-CBF-QP enforces safety only at \textit{selected} points $\by_i$ on the robot surface---leaving the question of how to constructively select these points to guarantee full-body safety (\ie including unsampled points). 
The next subsection addresses the first requirement while Section III addresses the second.

\subsection{Occupancy-Based Safety with Poisson Safety Functions}
While CBFs provide a framework for synthesizing safe control actions, their success relies on the availability of a continuously differentiable function $h$ characterizing safety specifications. In practice, $h$ must be derived from enviromental data encoding obstacle occupancy information in task space.
Since sensors sample continuous spatial structures of the physical environment, an occupancy map is best understood as a discrete approximation of an underlying continuous representation, whose fidelity depends on the grid resolution.
Poisson Safety Functions \cite{bahati2025dynamic} bridge this gap by providing a smooth function $h_0$ constructed directly from the occupancy map, enabling the systematic synthesis of CBFs.
\begin{definition}(Poisson Safety Function \cite{bahati2025dynamic})
Given occupancy information, let $\Oc \subset \re^3$ denote the open, bounded, and connected free space with smooth boundary $\pOc$ corresponding to obstacle surfaces. We call $h_0:\re^3 \rightarrow \re$ a Poisson Safety Function if it is the unique solution to the Dirichlet problem for Poisson's equation:
\begin{gather}\label{eq: poisson's eq}
\!\!\!\left \{
    \begin{aligned}
        \frac{\partial^2 h_0}{\partial x^2}(\by) + \frac{\partial^2 h_0}{\partial y^2}(\by) + \frac{\partial^2 h_0}{\partial z^2}(\by) &= f(\by)&  \forall \by \in \Omega,\\
        h_0(\by) &= 0 &  \forall \by \in\partial \Omega, \\
    \end{aligned}
    \right.
\end{gather}
where $f: \Oc \rightarrow \re_{<0}$ is a prescribed forcing function.  
\end{definition}
%
The boundary condition in \eqref{eq: poisson's eq} places the $0$-level set on obstacle surfaces, and $f(\by) <0$ ensures $h_0(\by) >0$ in free-space. A smooth forcing function $f \in C^\infty(\overline{\Omega})$ yields a smooth solution $h_0 \in  C^\infty(\overline{\Omega})$ \cite{gilbarg1977elliptic}. As demonstrated in \cite{bahati2025dynamic}, $h_0$ is a CBF for single integrator systems \eqref{eq: velocity model}, \ie $h_0 \coloneq h$, and its smoothness enables CBF extensions for high-order systems. 
Notably, \eqref{eq: poisson's eq} naturally extends to dynamic environments with moving obstacles $\Omega =\Omega(t)$, resulting in a moving boundary value problem solved in real time yielding $t \mapsto h_0(\by, t)$ \cite{bena2025geometry}. 


 PSFs resolve the differentiability requirement in \eqref{eq: task space gradients} enabling the CBF condition to be enforced classically in the QP. That is, a single PDE solve produces a \textit{globally} smooth CBF $h_0$ for environments with obstacles of arbitrary geometry, without requiring a separate barrier construction for each obstacle. Thus, in the multi-constraint safety filter (MC-CBF-QP), the same $h_0$ is simply evaluated at each sample point $\by_i \mapsto h_0(\by_i,t) = h_0(\varphi_i(\bq),t)$.

\begin{figure*}
\centering
    \includegraphics[width=1\linewidth]{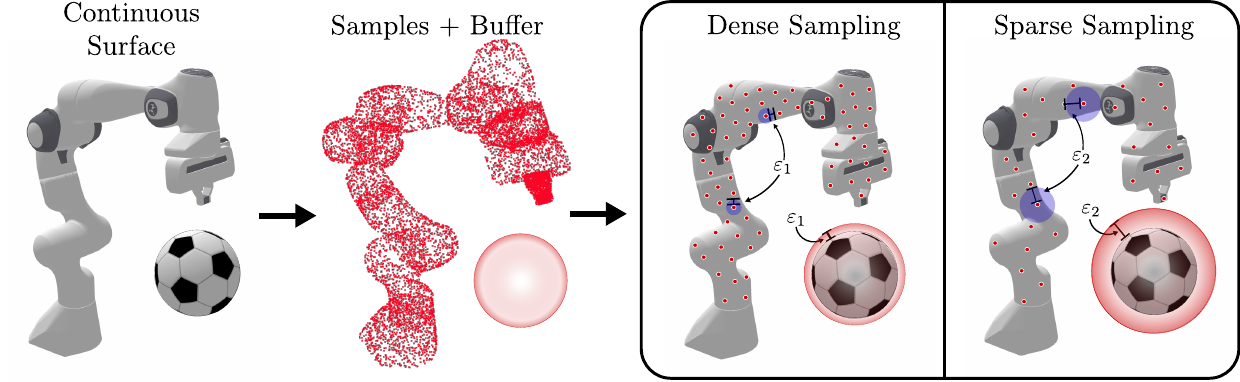}
    \caption{\small A dense point cloud offers an approximation of the robot's geometry. This is down-sampled via the Poisson Disk Algorithm to generate a sample set that contains all surface points in the union of $\varepsilon$ balls centered at each sample. The final two figures demonstrate the relationship between the sample radius $\varepsilon$ and the inflated buffered occupancy representation of an obstacle. }
    \label{pointcloud_to_sample}
    \vspace{-4mm}
    \label{fig: sampling and buffer main figure}
\end{figure*}

\section{Full-Body Safety For Robotic Manipulators}

Building on the foundations of CBF-based safety filters and PSFs, we now develop a constructive framework for full-body safety-critical manipulation in dynamic environments.
%
Specifically, we sample points on the manipulator surface at a specified resolution and buffer all obstacles in the environment according to that resolution. 
%
%
We prove that enforcing safety at the sampled points on the buffered map extends to the entire robot surface on the true (unbuffered) map, guaranteeing full-body collision avoidance with respect to the true obstacle boundaries.


\subsection{Sampling: Finite Approximation of the Robot Surface}

Since the robot presents a continuous surface, directly enforcing safety at every point on its body is computationally intractable. Instead, safety is enforced on a finite set of sample points, which serve as a discrete approximation of the robot surface.
To this end, let $\mathcal{S}(\bq) \subset \mathbb{R}^3$ denote the manipulator surface in task-space at configuration $\bq$, and let $\mathcal{Y} = \{\by_1, \dots, \by_N\} \subset \mathcal{S}(\bq)$ be a set of sample points approximating the surface at a sampling resolution $\varepsilon >0$.
To provide safety guarantees, we require the set $\Yc$ to satisfy a \textit{coverage} condition, where every point on the surface,  $\bp \in \mathcal{S}(\bq)$, lies within $\varepsilon$  of at least one sample point $\by_i \in \Yc$, expressed as:
%
%
%
\begin{align}\label{eq: coverage}
    \forall\, \bp \in \Sc(\bq), \ \ \min_{i \in \{1, \cdots, N\}} \| \bp - \by_i \| < \varepsilon.
\end{align}
Equivalently, \eqref{eq: coverage} implies the entire robot surface is contained in the union of $\varepsilon$-balls centered at each sample point.
\begin{align}
    \mathcal{S}(\bq) &\subset \bigcup_{i=1}^{N} B_\varepsilon(\by_i), \\
    B_\varepsilon(\by_i) \triangleq \by_i + B_\varepsilon &= \{\by \in \mathbb{R}^3 : \|\by - \by_i\| < \varepsilon\}, \label{eq: epsilon ball}
\end{align}
%
%
where $B_\varepsilon = \{\by \in \re^3 : \|\by\| < \varepsilon\}$ is the open  $\varepsilon$-ball centered at the origin.
This reveals that the union of the $\varepsilon$-balls is an over-approximation of the true robot surface. To achieve full-body collision avoidance, rather than enforcing safety directly on this over-approximation, we instead buffer the obstacles by $\varepsilon$, an equivalent and more tractable formulation, which we discuss in the next subsection.



To construct $\Yc$ in practice, we assume a dense point cloud approximating the robot surface is available, as in Figure~\ref{fig: sampling and buffer main figure}.
While this point cloud could serve directly as sample points naturally satisfying the coverage condition \eqref{eq: coverage} at a fine resolution $\delta \ll \varepsilon$, doing so yields a large number of CBF constraints, which becomes computationally prohibitive and potentially infeasible in real time. The point cloud, therefore, is used as an intermediate representation from which a sparser sample set is constructed.
%
In this work, we employ \textit{Poisson disk sampling}~\cite{bridson2007fast} to down-sample the point cloud by enforcing a minimum distance $\varepsilon$ between each sample. The algorithm retains a point only if no previously accepted points lies within this distance  $\varepsilon$, producing a sample set $\Yc$ that satisfies \eqref{eq: coverage} by construction, with approximately uniform surface coverage.

\subsection{Buffered Safe Region}

%
%
The coverage condition \eqref{eq: coverage} guarantees that every unsampled point on the surface is within $\varepsilon$ of a sample. Therefore, given a task-space safe set $\Cc \subset \re^3$ (\ie provided by an occupancy map defining free space), if every sample is kept at least $\varepsilon$ away from the safe boundary $\partial \Cc$, the entire $\varepsilon$-ball around each sample lies inside the safe region $\Cc$---and with it, every unsampled point that the ball covers, as illustrated in Figure~\ref{fig:full body safey and samples}. We enforce this margin by shrinking the safe region $\Cc$ via the Pontryagin difference \cite{bena2025geometry}:
%
%
\begin{align}\label{eq: buffered safe set}
    \Cc_\varepsilon = \Cc \ominus B_\varepsilon,
\end{align}
where $B_\varepsilon$ is the $\varepsilon$-ball centered at the origin \eqref{eq: epsilon ball} so that:
\begin{align}
    \by_i \in \Cc_\varepsilon \implies B_\varepsilon(\by_i) \subset \Cc.
\end{align}
Any point in $\Cc_\varepsilon \subset \Cc$ is at least $\varepsilon$ from the boundary $\partial \Cc$.


\begin{figure}[t]
\centering

\begin{minipage}[t]{0.52\linewidth}
    \vspace{0pt}
    \centering    \includegraphics[width=\linewidth]{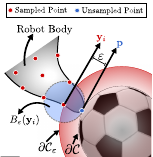}
\end{minipage}
\hfill
\begin{minipage}[t]{0.45\linewidth}
    \vspace{0.3cm}
    \centering
    \small
    \renewcommand{\arraystretch}{2.0}
    \setlength{\tabcolsep}{2pt}
    
    \begin{tabular}{c c c c}
    \hline
    $\varepsilon$ & $N$ & QP (s) & Buffer (s) \\
    \hline
    0.01 & 1699 & 0.0144 & 0.0002 \\
    0.05 & 116  & 0.0037 & 0.0003 \\
    0.10 & 28   & 0.0034 & 0.0007 \\
    0.20 & 10    &0.0034 & 0.0040 \\
     0.50 & 7    &0.0033 & 0.0542 \\

    \hline
    \end{tabular}
\end{minipage}

\caption{\small \textbf{(Left}) Relationship between buffering of the occupancy map by the sampling resolution $\varepsilon$ and full-body safety. \textbf{(Right)} Effect of sampling distance on computational performance, showing number of sample points $N$, QP solve times, and occupancy buffering times over a repeated dynamic experiment. Note $\varepsilon \geq 0.2$ produced some infeasibility due to the large buffering in the restricted workspace.}
\label{fig:full body safey and samples}

\vspace{-6mm}
\end{figure}

\subsection{Poisson Safety Function Synthesis in 3D}

Given the buffered safe region $\Cc_\varepsilon$, our next goal is to construct a time-varying continuously differentiable safety function $\he: \re^3 \times [0,T] \rightarrow \re$ for some $T>0$, whose $0$-superlevel set characterizes the safe set $\Cc_\varepsilon(t)$ for all $t \in [0,T]$.
%
%
%
To this end, we define a non-cylindrical domain $\Oc_{t}$ via the lifting operation \cite{bena2025geometry}:
\begin{align}
    \!\!\!\Occ_{t} \!= \!\bigcup_{t \in [0,T]} \Cce(t) \times \{t\} \subset\R^3\times [0,T],
\end{align}
allowing us to formulate the time parameterized Dirichlet problem for Poisson's equation in the space-time domain:
%
%
\begin{equation}
\label{eq: Q Poisson}
\left \{
    \begin{aligned}
        \frac{\partial^2 \he}{\partial x^2} + \frac{\partial^2 \he}{\partial y^2} + \frac{\partial^2 \he}{\partial z^2} &= f(\by,t) &\forall(\by,t) \in\Oc_{t}, \\
        \he(\by,t) &= 0  &\forall(\by,t) \in\partial\Oc_{t},
    \end{aligned}
\right.
\end{equation}

\noindent 
with $f(\by,t) < 0$ for all $(\by,t) \in\Oc_{t}$.
Solving \eqref{eq: Q Poisson} yields a PSF $(\by, t) \mapsto \he(\by,t)$ that characterizes $\Cce(t)$ according to:
\begin{equation}
\label{eq: Safe Set}
    \Cce(t) = \{\by\in\R^3 : \he(\by,t) \geq 0 \}.
\end{equation}
As demonstrated in \cite{bahati2025dynamic}, 
$\he$ is a CBF for the kinematic model \eqref{eq: velocity model}.
The following subsection establishes that the forward invariance of $\Cce(t)$ with respect to the sample points guarantees full-body collision avoidance on $\Cc(t)$. 


\subsection{Multi-Constraint Safety Filter}

Poisson's equation yields a single globally smooth solution $h_{\varepsilon}$ characterizing safety for the entire environment. Thus, at each time $t \in [0,T]$, each sample point $\by_i \in \Yc \subset \Sc(\bq)$ simply queries the same $h_\varepsilon$,
giving rise to a structured family of $N$ joint-space CBF constraints. Each safety constraint is derived from a single function, evaluated at different points on the robot body in the world frame, $h_\varepsilon(\by_i) \coloneqq h_\varepsilon(\varphi_i(\bq))$. Enforcing this family simultaneously yields the multi-constraint safety filter:
%
%
%
%
%
\begin{align*}
{\bv}_\text{safe} = &\argmin_{\bv \in \re^n} \quad \|{\bv}- {\bv}_{\nom}\|_2^2 \tag{MC-CBF-QP} \\
& \quad \mathrm{s.t.} \quad 
 \dot{h}_{\varepsilon}(\varphi_i(\bq), t,\bv)  \geq - \alpha_i h_{\varepsilon}(\varphi_i(\bq),t),  
\end{align*}
for all $i \in \{1, \cdots N\}$, where $\alpha_i > 0$, 
and 
$\frac{\partial h_\varepsilon}{\partial \bq}(\bq,t)
=\frac{\partial h_\varepsilon}{\partial \by}(\varphi_i(\bq),t) \cdot J_i(\bq)$.     
%
%
The multi-constraint safety filter (MC-CBF-QP) generates safe joint velocities that enforce the forward invariance of $\Cce$.
Forward invariance of $\Cc_\varepsilon$ under this filter ensures that all sample points $\by_i$ remain in $\Cc_\varepsilon$ for all time $t \in [0,T]$. 
In particular, the buffer $\varepsilon$ ensures safety enforced at the samples extends to the full continuous surface. Using the coverage condition, we demonstrate this is sufficient for full-body safety with the following theorem:

\begin{theorem}[Full-body Collision Avoidance]
Let $\mathcal{S}(\bq) \subset \mathbb{R}^3$ be the manipulator surface at configuration $\bq \in \re^n$ and let $\mathcal{Y} = \{\by_1, \dots, \by_N\} \subset \mathcal{S}(\bq)$ be sample points satisfying the coverage condition \eqref{eq: coverage} for some $\varepsilon > 0$, and
%
%
let $\Cc(t) \subset \re^3$ be a time-varying safe set defined by an occupancy map whose interior represents free space and boundary represents obstacle surfaces in dynamic environments. Consider an open ball $B_\varepsilon \subset \re^3$ of radius $\varepsilon$ as in \eqref{eq: epsilon ball} and the reduced safe set via the Pontryagin difference $\Cc_\varepsilon(t) = \Cc(t) \ominus B_\varepsilon$ as in \eqref{eq: buffered safe set}.
%
If the following condition holds:
\begin{align}\label{eq: thm assumption}
    \by_i \in \Cc_\varepsilon(t) \quad \forall \, i \in \{1, \dots, N\},
\end{align}
holds for all $t \geq 0$, then the entire robot surface satisfies:
\begin{align}
    \mathcal{S}(\bq) \subset \Cc(t) \quad \forall \, t \geq 0.
\end{align}
Hence, $\Cc(t)$ is rendered forward invariant with respect to the full robot body.
\end{theorem}

\begin{proof}
Fix any $t \geq 0$ and any point $\bp \in \mathcal{S}(\bq)$. By coverage, there exists $\by_i \in \mathcal{Y} \subset\mathcal{S}(\bq) $ with $\| \bp - \by_i \| < \varepsilon$. Since $\by_i \in \Cc_\varepsilon(t) = \Cc(t) \ominus B_\varepsilon$ from the condition \eqref{eq: thm assumption}, by definition of the Pontryagin difference, we have $B_\varepsilon + \by_i  = B_\varepsilon(\by_i) \subset \Cc(t)$. That is, every point within 
distance $\varepsilon$ of $\by_i$ lies in $\Cc(t)$. It then follows from $\| \bp - \by_i \| < \varepsilon$ that $\bp \in \Cc(t)$. As $\bp$ and $t$ were arbitrary, this implies $\mathcal{S}(\bq) \subset \Cc(t)$ for all $t \geq 0$ .
\end{proof}


The above theorem establishes full-body collision avoidance as illustrated in Figure \ref{fig:full body safey and samples}. The sampling resolution  $\varepsilon$ provides a tradeoff between computational efficiency and geometric fidelity---a larger $\varepsilon$ results in fewer CBF constraints but a coarser geometric approximation of the robot surface, and consequently, a larger safety buffer.
%
%
We highlight the relationship between the number of constraints and solve times in Figure \ref{fig:full body safey and samples}; sampling resolutions $\varepsilon \in [0.05, 0.1]$ are used in practice for reliability and real-time solvability.

\begin{remark}(Additional Constraints)
In practice, while the MC-CBF-QP provides safety guarantees against collision avoidance in the environment, increased performance may be desired. For example, with a three dimensional end effector control objective, joint-space velocity commands may not accurately capture the task goal. In other works, the objective is altered to capture Cartesian proximity via Jacobians and component weighting, as in \cite{morton2025safe}. Instead of tuning parameters directly, we leverage a Control Lyapunov function (CLF) to add a further constraint to the MC-CBF-QP. As it is not required for safety, the CLF is added as a soft constraint to bias the filtered command to move in a desired direction \cite{ames_2013_towards}. We define our CLF to keep the end-effector near a desired position:
$V(\bx) = \frac{1}{2}\lVert\bx - \bx_d\rVert^2$,
with $\bx, \bx_d \in \mathbb{R}^3$ the current and desired end-effector position based off the nominal command. Thus, we introduce the constraint $\dot{V}(\bx) \leq -\gamma V(\bx) + \delta$, with $\delta$ a linear optimization variable in the objective function. This could be replaced with another CLF designed to preference the user's priorities for the given task. Additionally, joint limits, encoded as half-space constraints $h_j^+(\bq) = \bq_j - \bq_{j,\min}$ and $h_j^-(\bq) = \bq_{j,\max} - \bq_j$, are appended directly to the QP at negligible computational cost, which prevent the robot from being driven to a joint angle out of its range. Further constraints could be added to avoid all forms of self collision  \cite{SingletaryMolnarSafetyCriticalFood}.

\end{remark}

\section{Hardware Experiments}

\subsection{Experimental Setup}

\begin{figure*}[t!]
\centering
    \includegraphics[width=\textwidth]{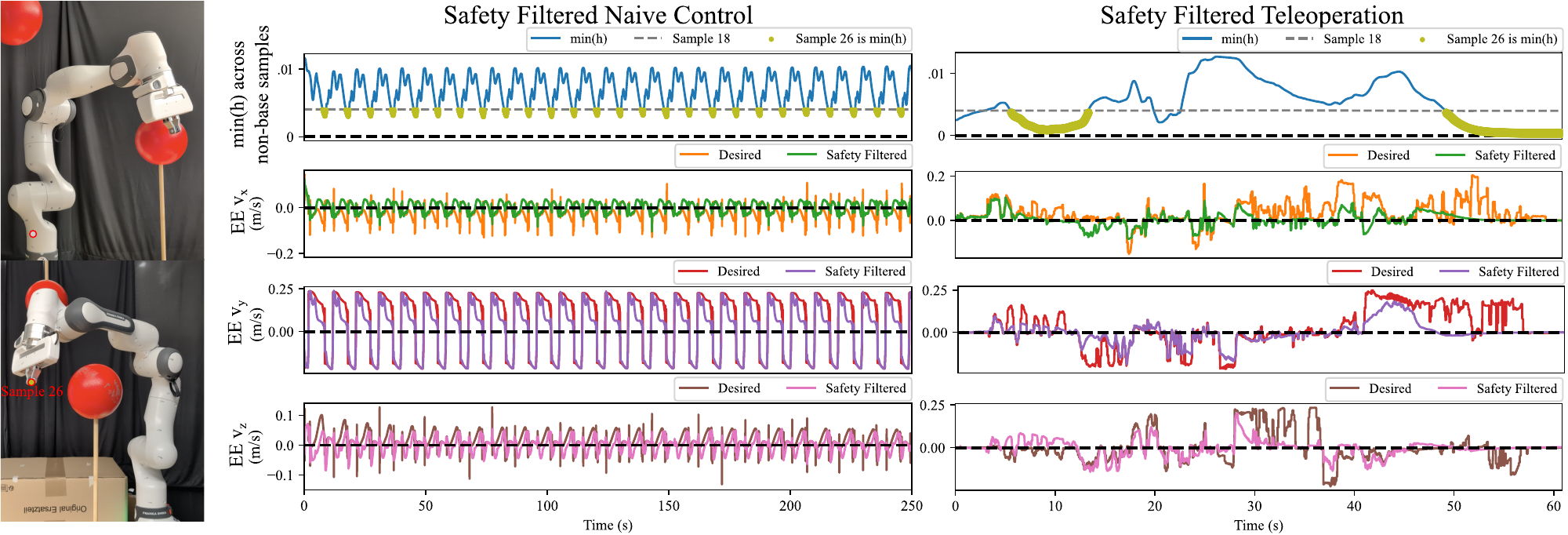}
    \caption{\small Sample experimental runs with a static environment consisting of boxes and spheres. \textbf{(Left)} Front \& side view of sample pose, with key body samples labeled. Sample 18, on the base, is close to the boundary, and has a consistently low $h$ value. Sample 26 attains the lowest $h$ value across both experiments, and is also highlighted. \textbf{(Top)} Minimum $h$ across non-base samples, and highlights when sample 26 has the lowest $h$. \textbf{(Bottom)} Commanded vs. filtered end effector velocities in $\mathbb{R}^3$ (transformed using desired and filtered joint velocities using forward kinematics). These plots show that filtering keeps the entire body safe, while not being overly restrictive of motions. }
    \label{fig:static_fig_control_filters}
    \vspace{-4mm}
\end{figure*}

We validate the proposed framework on a Franka Emika FR3, a seven degree of freedom (7-DOF) robotic arm, in both static and dynamic environments. The safety filter operates on joint velocity commands at $50 - 100$ Hz. The FR3 workspace is limited to a cube spanning $(-1.0, 1.0) \cross (-1.0, 1.0) \cross (0.0, 2.0)$ meters with the origin at the FR3 base frame. This workspace is uniformly discretized into a $100 \cross 100 \cross 100$ voxel grid.

Body surface points are obtained from the URDF geometric approximation of the FR3, visualized in Fig. \ref{fig: sampling and buffer main figure} \cite{todorov2012mujoco}. The sample set $\Yc$ is chosen as described in Section $3$. The sampling distance used for static and dynamic experiments was $\varepsilon = 0.1$ resulting in $30$ sample points across the robot body. Each sample point is defined in the local link frame, with base frame representation calculated at each control step via forward kinematics from the FR3's joint configuration. 

\begin{figure*}[!t]
\centering
\includegraphics[width=\textwidth]{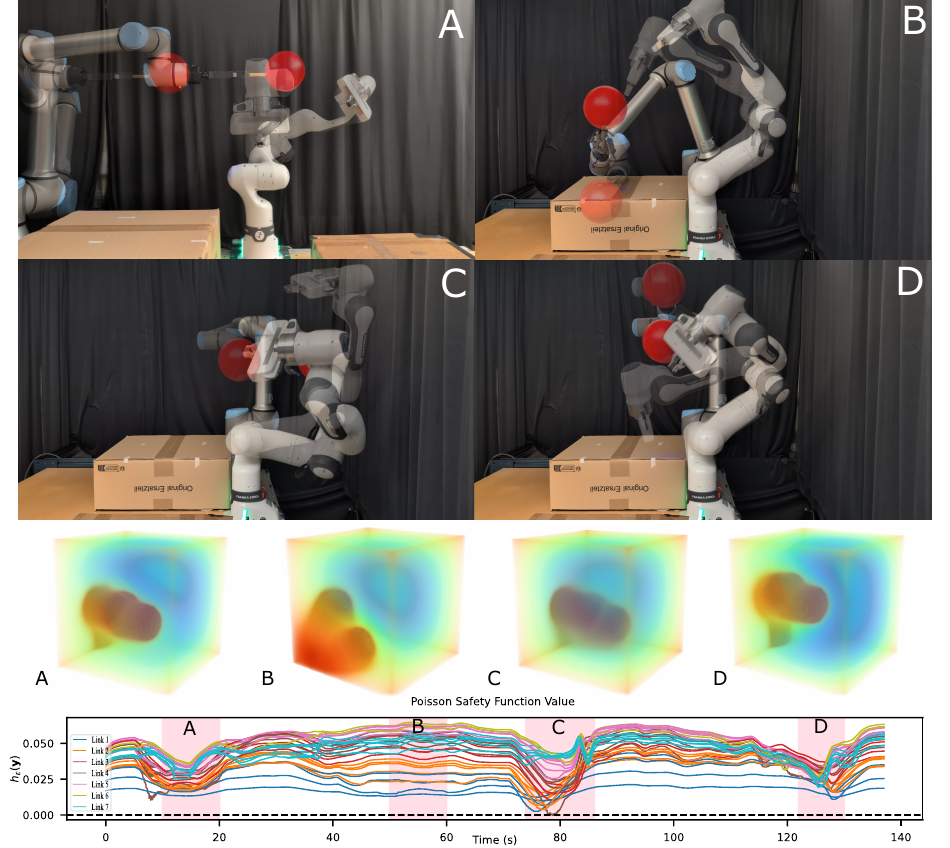}
\caption{\small Hardware validation in a dynamic environment. \textbf{(Top)} We teleop the UR10e  into the FR3's workspace. Panels A,B,C and D show different movements along the experiment.\textbf{(Middle)} Poisson Safety function $h_\varepsilon$. The heatmaps depict the 3D Poisson safety Function at the highlighted times. \textbf{(Bottom)} The values of the Poisson safety function evaluated at all $30$ sample points; it remains strictly positive demonstrating the safe behavior of the FR3 in an unpredictable, dynamic environment. Note the plot is grouped by the link number, and observe how different sample points decrease dependent on the UR10e's movements. }
\label{dynamic_fig}
\vspace{-6mm}
\end{figure*}
The occupancy grid over which Poisson is solved must encode only fully free voxels at least $\varepsilon + \delta$ away from obstacles as open. An initial grid, with obstacles fully contained in the occupied regions, is consequently buffered by $\varepsilon + \delta$, with $\delta$ negligible in the following experiments. Buffering can be parallelized and only requires linear complexity, and is thus not a computational bottleneck for reasonable $\varepsilon$, see Figure \ref{fig:full body safey and samples}. The Poisson safety function is computed using a successive over relaxation (SOR) method \cite{young1950iterative} with an alternating checkerboard scheme. Coupled with a warm start from the previous solve, the numerical PDE solve time is on average $0.002$ s across static and dynamic environments.  The occupancy map composition, buffering, and Poisson safety filter synthesis are executed on a NVIDIA RTX GeForce 5090. 
The multi-constraint quadratic program was solved, on average, in $0.003$s using OSQP \cite{stellato2020osqp}. 

The state measurements utilized in the Poisson synthesis introduce uncertainty due to latency and encoder error. To account for this, we enforce an input-to-state safe control barrier function (ISSf - CBF) in our MC-CBF-QP \cite{kolathaya_2019_inputtostate}. We introduce an adaptive robustness term into the CBF constraint $\frac{\partial \he}{\partial \by} + \frac{\partial \he}{\partial t} + \alpha_i h_\varepsilon \geq \varepsilon_0\|\frac{\partial \he}{\partial \by}\|^2$, with $\varepsilon_0 > 0 $; this accounts for uncertainty without over-conservatism 
\cite{da_2026_safe}.
\subsection{Static Environment}
Robotic manipulators are often required to perform tasks in cluttered environments \cite{morgan2019benchmarking}. When executing pick-and-place operations or more generally manipulating objects, robots must safely navigate such environments while executing a nominal closed-form controller or deploying a learned policy \cite{articlediff}. We thus evaluate our proposed framework in a range of static environments and controllers. The obstacles are assumed to be known, prescribed a priori or obtained from perception data. As the obstacles are fixed, the occupancy grid and resulting Poisson Safety function are obtained once at initialization and queried throughout. 

In static experiments, we test our safety filter on both naive and adversarial controls. For adversarial commands, the FR3 is teleoperated via joint velocity commands at $100$Hz to track desired end effector cartesian positions. The user commanded nominal controller generates trajectories into obstacles in the environment, attempting to violate safety and drive the robot into unsafe regions where $h(\mathbf{y}_i) \leq 0$ at one or more sampled points $\by_i$. These nominal commands are passed through the safety filter, returning safe joint velocity commands. In Figure \ref{fig:static_fig_control_filters}, the minimum across $h(\by_i)$ is plotted alongside desired vs. filtered end effector velocity commands. It is clear that the Poisson Safety function is strictly positive for the sampled set of points and preserves safety of the entire robotic arm in spite of adversarial commands. Additionally, a simple controller
pushes the end effector to visit points A and B in the robot workspace. Over 25 back-and forth-traversals are completed in under 4 minutes while maintaining safety, showing the filtering is not overly conservative. 

\subsection{Dynamic Environment}

During manipulation tasks, beyond avoiding the static environment, it is desired that the robot can react safely to dynamic objects outside the system's control, with unknown intentions, such as a human or another robot entering the workspace. We validate our proposed safety filter framework in a live, dynamic environment and assess the real time performance of the full-body collision avoidance.

A spherical object is rigidly attached to a UR10e, a 6-DOF robotic arm, to serve as the dynamic obstacle for the FR3. The UR10e is teleoperated, freedriven, or executing a nominal controller, in real time and guided into the FR3 workspace to induce potential collisions, and the FR3 is only considered safe if every link on the FR3 does not collide with any portion of the UR10e or its spherical object. The dynamic motion is unknown to the FR3 beyond state measurements, and is not pre-planned during teleoperation or freedrive mode. The UR10e joint state is sampled at $100$Hz and, using forward kinematics and a simplified robot model, is used to build the occupancy grid as the UR10e moves throughout the workspace. The occupancy grid thus contains the entire UR10e geometry as well as the attached ball. During each control cycle the occupancy grid is updated according to the current state and the Poisson safety function is computed online, in real time, at an average of $100$Hz. In these experiments, the FR3's goal is to maintain a fixed joint position, using proportional control, which the safety filter augments to send safe joint velocity commands.

\begin{figure}
    \centering
    \includegraphics[width=\linewidth]{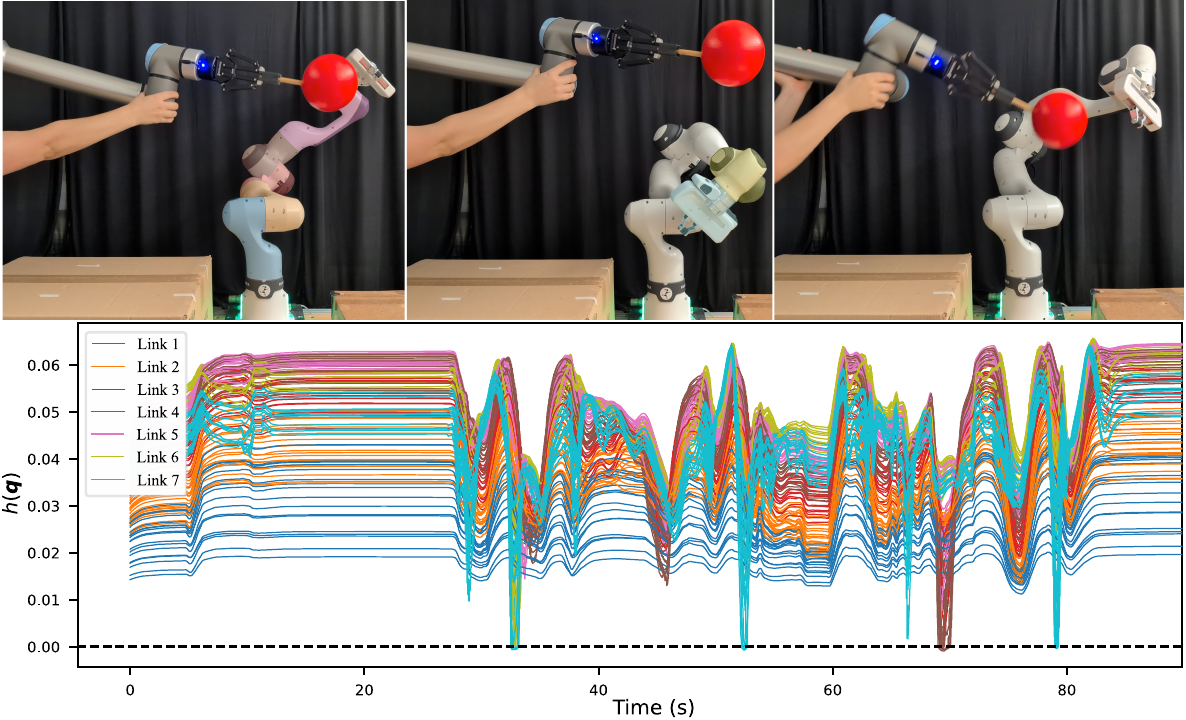}
    \caption{\small In freedrive,  a human moves the UR10e at higher speeds into the FR3's workspace. In this run $\varepsilon = 0.05$ with $121$ samples on the robot's body, colored by link. The FR3 is able to maintain safety of the entire sample set as well as demonstrate effective collision avoidance in unpredictable human-driven obstacle motions. }
    \label{fig:human_freedrive}
    \vspace{-6mm}
\end{figure}
The UR10e was commanded into the FR3's workspace from a range of positions and angles to fully test the full body collision avoidance. Across these trials, the FR3 was able to execute full body collision avoidance in real-time. Figure \ref{dynamic_fig} shows an example obstacle trajectory teleoped live and the FR3 is able to successfully avoid the obstacle with the Poisson safety function positive consistently across sample points. It can also handle faster, more unpredictable motions, as seen in Figure \ref{fig:human_freedrive}, where a human drives the UR10e into the workspace. Note effective collision avoidance is restricted  to bounded obstacle velocities relative to the control rate and FR3 speed and acceleration limits. Additionally, given that the FR3 is a serial-chain manipulator, initial links lack sufficient degrees of freedom to guarantee avoidance. For example, no control action can keep the base safe from obstacles approaching the base-frame origin.


\section{Conclusion}

In this work, we have presented a framework for full-body collision avoidance for manipulation.
By sampling the robot surface at a prescribed resolution and buffering the environment accordingly, a single smooth Poisson Safety Function, constructed directly from onboard sensing, yields a structured family of CBF constraints suitable for real-time safety filtering.
We validate the proposed framework on a $7$-DOF manipulator in both static and dynamic environments. Future work includes dynamically synthesizing occupancy maps from 3D perception data, incorporating semantic reasoning for context-aware safety applications and adaptive sampling techniques for improved computational efficiency. 

\vspace{2mm}
\noindent\textbf{Acknowledgments} - This research was supported by The Dow Chemical Company through project \#227027AW.
ChatGPT 5.2 was used as an aid to develop helper functions in the code.

\addtolength{\textheight}{-12cm}   





\bibliographystyle{IEEEtran}

\bibliography{main-GB}

\end{document}